\documentclass[11pt]{article}
\usepackage{fullpage}% DO NOT CHANGE THIS
\usepackage{times} % DO NOT CHANGE THIS
\usepackage{helvet} % DO NOT CHANGE THIS
\usepackage{courier} % DO NOT CHANGE THIS
\usepackage[hyphens]{url} % DO NOT CHANGE THIS
\usepackage{graphicx} % DO NOT CHANGE THIS
\usepackage{graphicx} % DO NOT CHANGE THIS
\usepackage[utf8x]{inputenc}
\usepackage{caption}
\usepackage[dvipsnames]{xcolor}
\usepackage{hyperref}
\hypersetup{
	colorlinks=true,
	linkcolor=magenta,     
	citecolor= ForestGreen
}
\usepackage{adjustbox}
\usepackage{times}
\usepackage{amsmath,amsthm,amssymb}%,amsbsy}
\newtheorem{lem}{Lemma}
\newtheorem{thm}{Theorem}%[section] %(If you want theorem numbered
\theoremstyle{remark}

\theoremstyle{remark}
\usepackage{enumerate}
\usepackage{graphicx}
\usepackage{subfigure}
\usepackage{booktabs}
\usepackage{subfiles}
\usepackage{algorithm}
\usepackage{algorithmic}
\usepackage{xspace}
\makeatletter
\DeclareRobustCommand\onedot{\futurelet\@let@token\@onedot}
\def\@onedot{\ifx\@let@token.\else.\null\fi\xspace}

\def\ie{\emph{i.e}\onedot} 
 
\def\etc{\emph{etc}\onedot} 
\def\wrt{w.r.t\onedot} 
\def  \vv {\boldsymbol{v}}

\newcommand{\R}{\mathbb{R}}

\def \real    { \mathbb{R} }

\def \Real {\mathbb{R}}

\newcommand{\e}{\begin{equation}}
\newcommand{\ee}{\end{equation}}
\newcommand{\en}{\begin{equation*}}
\newcommand{\een}{\end{equation*}}
\newcommand{\eqn}{\begin{eqnarray}}
\newcommand{\eeqn}{\end{eqnarray}}
\newcommand{\bmat}{\begin{bmatrix}}
	\newcommand{\emat}{\end{bmatrix}}
\newcommand{\vct}[1]{\boldsymbol{#1}}
\newcommand{\mtx}[1]{\boldsymbol{#1}}
\def \lg        {\langle}
\def \rg        {\rangle}

\newcommand{\set}[1]{\mathbb{#1}}
\DeclareMathOperator*{\argmin}{\text{arg~min}}
\DeclareMathOperator*{\argmax}{\text{arg~max}}

\newcommand{\vq}{\vct{q}}

\newcommand{\vu}{\vct{u}}

\newcommand{\vx}{\vct{x}}

\newcommand{\mB}{\mtx{B}}

\newcommand{\mJ}{\mtx{J}}

\newcommand{\mM}{\mtx{M}}

\newcommand{\mU}{\mtx{U}}
\newcommand{\mV}{\mtx{V}}
\newcommand{\mW}{\mtx{W}}
\newcommand{\mX}{\mtx{X}}
\newcommand{\mY}{\mtx{Y}}
\newcommand{\mZ}{\mtx{Z}}

\newcommand{\mSigma}{\mtx{\Sigma}}

\newcommand{\mId}{{\bf I}}

\newcommand{\setU}{\set{U}}
\newcommand{\setV}{\set{V}}

\graphicspath{{./figs/}}
\def \lg        {\langle}
\def \rg        {\rangle}
\newlength{\imgwidth}
\newboolean{twoColVersion}
\setboolean{twoColVersion}{false}
\newcommand{\twoCol}[2]{\ifthenelse{\boolean{twoColVersion}} {#1} {#2} }

\newcommand{\optr}[1]{\operatorname{\textbf{tr}}(#1)}

\begin{document}
\title{A Unified Framework for Spherical Matrix Factorization}
\author{Kai Liu~\thanks{Computer Science Division, Clemson University. \texttt{kail@clemosn.edu}}
%	\\
%	\and
%	Yarui Cao~\thanks{Computer Science Division, Clemson University. \texttt{yaruic@clemson.edu}}
}
\date{}
\maketitle

\begin{abstract}
	Matrix Factorization plays an important role in machine learning such as Non-negative Matrix Factorization, Principal Component Analysis, Dictionary Learning, \etc However, most of the studies aim to minimize the loss by measuring the Euclidean distance, though in some fields, angle distance is known to be more important and critical for analysis. In this paper, we propose a method by adding constraints on factors to unify the Euclidean and angle distance. However, due to non-convexity of the objective and constraints, the optimized solution is not easy to obtain. In this paper we propose a general framework to systematically solve it with provable convergence guarantee with various constraints. %Experiments on synthetic data and real-world datasets have validated the effectiveness of our method and demonstrated its advantages over state-of-art clustering methods.
\end{abstract}

\section{Introduction}
Principal Component Analysis (PCA) is widely known to be one of the most popular methods for dimension reduction. Mathematically, assume that we have a set of $n$ sample images $\mX = [\vx_1, \vx_2, \cdots, \vx_n]\in \Real^{m\times n}$, where $\vx_i\in \Real^m$ and $i \in[1,n]$ denotes the $i$-th data (WLOG, we assume the data is centralized). The objective of PCA is to minimize the  error between original data and reconstruction under new base $\mW \in \textbf{R}^{m\times r}$ ($r$ denotes the reduced dimension):
\begin{equation}
	\label{eq:reconstruction}
	\min_{\mW} h=\|\mX - \mW\mV\|_F^2, \quad  s.t.\quad \mW^T\mW = \mId,
\end{equation}
where $\|\cdot\|_F$ denotes the Frobenius norm of a matrix. Apparently, if there is no constraint on $\mV$, then by taking the derivative of $h$ \wrt $\mV$ and set to be 0, then $\mV^*=\mW^T\mX$. $\mW$ contains the principal directions and $\mV$ indicates the principal components (data projects
along the principal directions). However, when it is to optimize sparse components, namely $\mV$ is imposed with sparsity constraint, $\mW^T\mX$ is no longer the optimal solution~\cite{liu2019spherical}.

Another example is Nonnegative Matrix Factorization (NMF), which aims to find two nonnegative matrices, $\mU \in \real^{m\times r}_+$ and $\mV \in \real^{r\times n}_+$, whose product can best  approximate an input nonnegative data matrix $\mX \in \real_+^{m \times n}$, \ie, $\mX \approx \mU\mV$. Usually, one can interpret the columns of $\mX$ as data points and the rows of $\mX$ as observations (features). A broadly used objective to learn NMF is to minimize the following objective:
\begin{equation}
	\label{eq:nmf_obj}
	\min_{\mU,\mV\geq 0}\ h(\mU,\mV)=\frac{1}{2} ||\mX-\mU\mV||^2_F.
\end{equation}
 If we consider the $r$ columns of $\mU$ the basis vectors, every column of $\mV$ approximates the corresponding data point in $\mX$ by a linear combination of these $r$ bases vectors, where the elements of the column of $\mV$ specify the coefficients to compute the linear combination. NMF has been found useful in a large variety of real-world applications such as image feature extraction~\cite{lee1999learning,liu2019study}, document clustering~\cite{xu2003document,liu2015robust,liu2018high}, single speech separation~~\cite{schmidt2006single}, music transcription~\cite{smaragdis2003non}, bioinformatics~\cite{lee2001application,rossini1994non,potluru2008group,liu2018multiple}, recommendation system~\cite{yu2016dynamic}, astronomy~\cite{blanton2007k,berne2007analysis}, to name a few.

To solve the NMF objective in Eq.~\eqref{eq:nmf_obj}, a Multiplicative Updating Algorithm (MUA) was derived using the following updating rules \cite{lee2001algorithms}:
\begin{equation*}
	%\label{eq:updata_MUA}
	\mV_{ij} \leftarrow \mV_{ij}\frac{(\mU^T\mX)_{ij}}{(\mU^T\mU\mV)_{ij}}, \quad \mU_{ij} \leftarrow \mU_{ij}\frac{(\mX\mV^T)_{ij}}{(\mU\mV\mV^T)_{ij}}.
\end{equation*}
The convergence of this algorithm was proved using the auxiliary function method \cite{lee2001algorithms}, whose correctness was also analysed in \cite{ding2010convex}.

Following the updating algorithm listed above, many NMF-based learning methods have been proposed. For instance, to guarantee the uniqueness of $U$ and $V$, orthogonal constraints were used in ~\cite{ding2006orthogonal}:
\begin{equation}
	\label{eq:nmf_orthoobj}
	\min_{\mU, \mV\geq 0}\ h(\mU,\mV)=\frac{1}{2} ||\mX-\mU\mV||^2_F, \quad s.t. \ \ \mU^T\mU=\mId.
\end{equation}
To include the mixed signs of the data matrix, Semi-NMF objective was studied in \cite{ding2010convex}:
\begin{equation}
	\label{eq:nmf_semiobj}
	\min_{\mV\geq 0}\ h(\mU,\mV)=\frac{1}{2} ||\mX-\mU\mV||^2_F.
\end{equation}

Inspired by the examples illustrated above and more (such as Dictionary Learning, \etc), it is noticeable that matrix factorization plays an important role in classical machine learning. Different constraints will lead to various names, but the nature is similar: to approximate the data with linear combinations of learned basis (can be orthogonal such as in PCA and orthogonal dictionary learning).
%Despite their successes, these methods all suffered from drawbacks of MUA, such as being easily trapped into suboptimal local minima~\cite{yanez2017primal,lin2007projected}, no strictly orthogonal solution $U$ can be obtained~\cite{liu2018high}, unknown sequence ($\{U^1, U^2, \dots, U^k\}$,  $\{V^1, V^2, \dots, V^k\}$) convergence property (sub or global) as well as the convergence rate~\cite{lin2007convergence}, and so on.

\begin{figure*}%[t!]
	\centering
	\includegraphics[width=.9\linewidth]{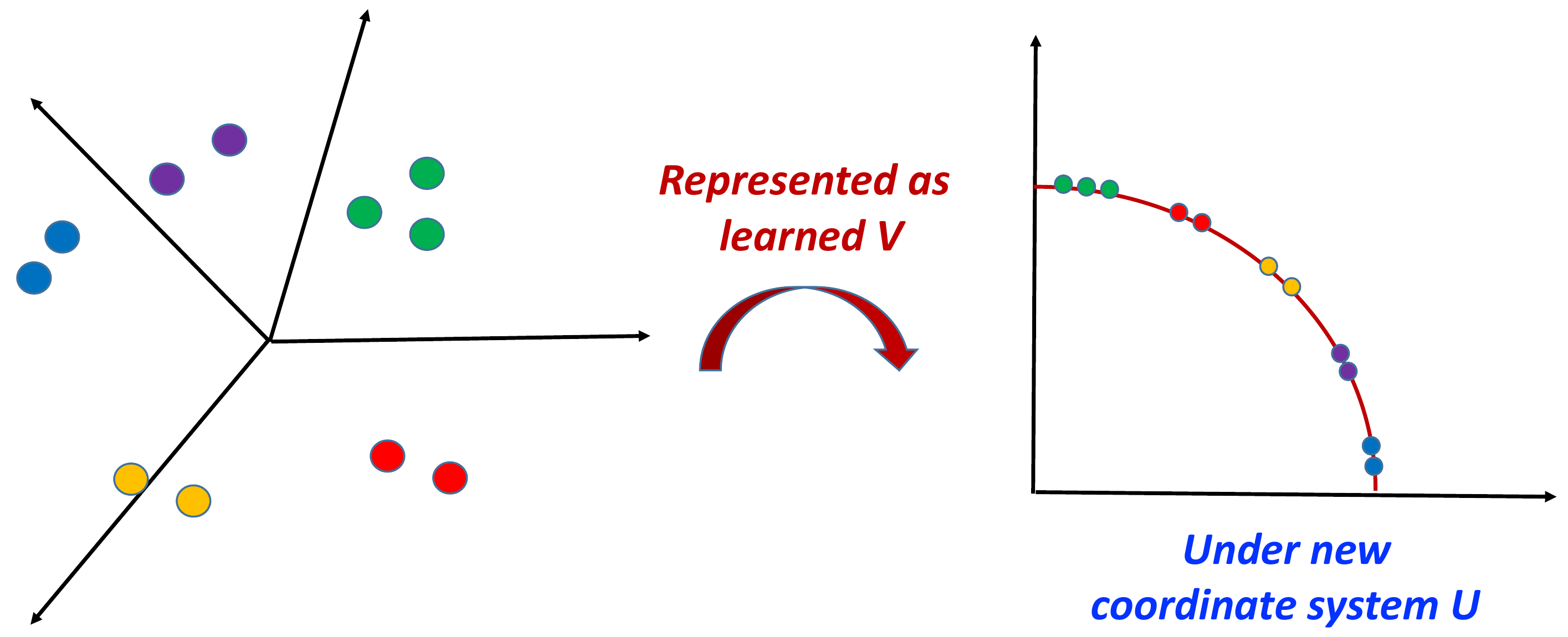}
	\caption{Data in original high dimension ($m=4$) can be analysed (such as clustered or classified) in lower dimension ($r=2$) space, where all data points lie in the sphere under new coordinate system. Since we have nonnegative constraint on $V$, they all lie in the first quadrant.}
	\label{fig:theta1}
\end{figure*}
\begin{figure}%[htp]
	\centering
	\includegraphics[width=0.4\linewidth]{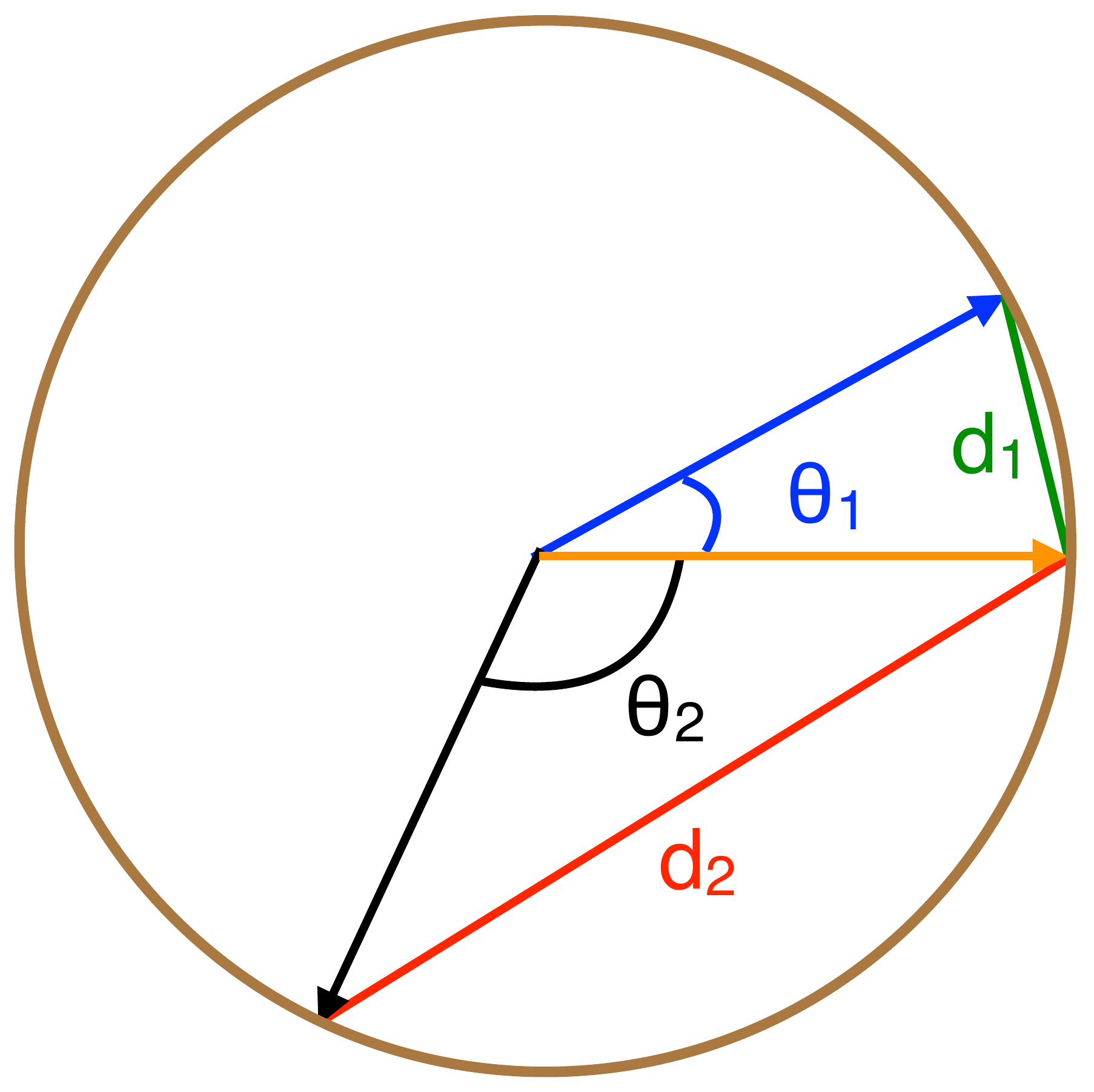}
	\caption{Larger angle distance ($\theta_2>\theta_1$) in the sphere will have larger Euclidean distance ($d_2>d_1$), and vice versa, which unifies the cosine similarity and Euclidean distance.}
	\label{fig:theta}
\end{figure}

\section{Motivation and Our Contributions}
In Eq.~(\ref{eq:reconstruction}--\ref{eq:nmf_semiobj}), the objectives minimize  difference between original data $\mX$ and approximated reconstruction data $\mU\mV$, which is measured by squared Euclidean distance and treat each feature equally important. However, in real world applications, there exist  datasets where distance-based measurement method may yield poor results~\cite{tunali2016improved}. In contrast, similarity-based measurements, such as angle distance, have been found wide successful applications in information retrieval~\cite{singhal2001modern,dhillon2001concept}, signal processing~\cite{hou1987fast}, metric learning~\cite{nguyen2010cosine}, and so on. Although one can calculate the similarity (cosine distance) from the vectors in $\mV$, after all it is an indirect method which is usually not effective. Therefore, deriving a direct method that can straightforwardly measure angle distance from Matrix Factorization (MF) is important and urgent.
%\begin{figure}
%	\centering
%	\includegraphics[width=0.7\linewidth]{theta2_d2}
%	\caption{Larger angles ($\theta_2>\theta_1$) in the sphere will have larger Euclidean distance ($d_2>d_1$), and vice versa, which unifies the cosine similarity and Euclidean distance simultaneously.
%	\label{fig:theta}
%\end{figure}

%\begin{figure*}[h!]
%	\centering
%	%\label{fig:compare}
%	\subfigure{
%		\centering
%		\begin{minipage}[b]{0.48\textwidth}
%			\includegraphics[width=.7\textwidth]{theta2_d2}
%			%\includegraphics[width=1\textwidth]{fig2.eps}
%		\end{minipage}
%	}
%	\subfigure{
%		\centering
%		\begin{minipage}[b]{0.48\textwidth}
%			\includegraphics[width=.75\textwidth]{c123_3}
%		\end{minipage}
%	}
%	\caption{\textbf{Left}: Larger angles ($\theta_2>\theta_1$) in the sphere will have larger Euclidean distance ($d_2>d_1$), and vice versa, which unifies the cosine similarity and Euclidean distance simultaneously. %\textbf{Center}: clustering result with distance -based method $K$-means.
%		 \textbf{Right}: Clustering data (after projection) distributed on sphere. Data with close angle have smaller Euclidean distance as well and will be clustered into the same clustering ($C_1, C_2, C_3$ denote different clusters).}
%	\label{fig:theta}
%\end{figure*}

With the above motivation, in this paper we propose a general framework for \textbf{Spherical Matrix Factorization} model which unifies Euclidean and angle distance. As illustrated in Fig.~\ref{fig:theta1}, we assume that each data point in $\mX$ that reside in the original high-dimensional space ($\real^m$) can be represented as $\mV$ in a lower dimensional space  ($\real^r$) under new coordinate system $\mU$. By noticing that larger angle in the sphere in Fig.~\ref{fig:theta} also has larger Euclidean distance, we can add the normalization constraint to the component matrix $V$ to guarantee the spherical distribution of components:
\begin{align}
	\min_{\mU\in\R^{m\times r},\mV\in\R^{r\times n}} h &= \|\mX-\mU\mV\|^2_F 
	= \sum_{i,j}[\mX_{ij}-(\mU\mV)_{ij}]^2 \nonumber\\
	s.t.&~~~~ \mU \in \setU, \mV\in \setV,\label{eq:nmf_obj_pca3}
\end{align}
where $\setU=\setU_1 \cup \setU_2, \setV=\setV_1 \cup \setV_2\cup \setV_3\cup \setV_4$ and
\e
\begin{aligned}
&\setU_1:=\{\mU \ | \  \mU^T\mU = \mId\}, \setU_2:=\{\mU \ | \  \mU\ge 0\};\\
&\setV_1:=\{\mV \ | \ \|\mV(:,j)\|=l, \ \forall j\}, \setV_2:=\{\mV \ | \ \mV \ge 0,  \|\mV(:,j)\|=l, \ \forall j\},\\
&\setV_3:=\{\mV \ | \ \|\mV(:,j)\|=l, \|\mV(:,j)\|_0\le s, \ \forall j\}, \setV_4:=\{\mV \ | \ \mV \ge 0,  \|\mV(:,j)\|=l, \|\mV(:,j)\|_0\le s, \ \forall j\}.%\in [1,n] \}.
\label{eq:setFG 3}
\end{aligned}
\ee
One can see that when $\mU\in \setU_2, \mV\in\setV_2$, it is Spherical Non-negative Matrix Factorization;  when $\mU\in \setU_1, \mV\in\setV_1$, it is Spherical PCA, and so on. Suppose the component is spherically distributed (with $l$ representing sphere radius), then the Euclidean distance between $\vv_i$ and $\vv_j$ is:
\begin{align}
	||\vv_i-\vv_j||^2_2 = &||\vv_i||^2+||\vv_j||^2-2\langle
	\vv_i, \vv_j\rangle
	\label{eq:nmf_obj_pca4}
	\\
	=&  ||\vv_i||^2+||\vv_j||^2-2\frac{\langle
		\vv_i, \vv_j\rangle
	}{||\vv_i||||\vv_j||} \\=& 2l^2-2 \ cos(\theta), \theta \in [0,\pi],\nonumber
\end{align}
which illustrates that larger angle $\theta$ will yield larger Euclidean distance, and vice versa.

%\begin{remark}
% $x_i \neq Uv_i$ and usually $\|x_i\| \neq \|Uv_i\|$ (they may be equal, but it barely happens) . We have $||x_i||^2=1$ for normalized data and if $||v_i||^2=1$ then $||Uv_i||^2=\tr(v_i^TU^TUv_i)=\tr(v_i^Tv_i)=||v_i||^2=1$, which leads a contradiction, thus the constraint on $V$ is necessary to guarantee our motivation.
%\end{remark}

In the rest of this paper, we will propose a general but efficient framework to solve Eq.~\eqref{eq:nmf_obj_pca3} where constraints are varied in Eq. (\ref{eq:setFG 3}). Here we summarize our contributions explicitly as follows:
\begin{itemize}%\itemsep-0.4em
	\item We propose a matrix factorization model that unifies Euclidean and angle distance, which can work well on various datasets.
	\item We systematically propose a wild framework which will make the objective monotonically decreasing while strictly satisfying constraints.
	%\item We extend NMF to Semi-NMF to include mixed signs data, while clustering results can still be determined by nonnegative membership matrix $V$.
	\item Our algorithm not only can guarantee the monotonically decreasing property of the objective, but also can guarantee the global-sequence convergence property of $\mU$ and $\mV$ with at least sub-linear convergence rate.
	%\item We give a rigorous proof for the convergence rate of our method, which is more efficient than MUA.
\end{itemize}

\section{Formulation And Algorithm}
%\subsection{Objective Function with Proximal Term}
We first denote:
\e
%\begin{aligned}
\label{eq:h}
h(\mU,\mV) = \|\mX-\mU\mV\|_F^2, \quad s.t. \quad \mU \in \setU, \mV\in \setV.
%\end{aligned}
\ee
By noting that Eq.~(\ref{eq:nmf_obj_pca3}) is non-convex, for which no closed solution exists, a natural idea is to use  alternating minimization method to solve the optimization problem:
\e
\begin{aligned}
	\label{eq:alm}
	\mU_{k+1} &= \argmin \|\mX-\mU\mV_k\|_F^2, \quad \mU\in\setU,\\
	\vv_{k+1} &= \argmin \|\vx-\mU_{k+1}\vv\|_F^2, \quad \mV\in\setV,
\end{aligned}
\ee
where $\vv$ denotes the column of $\mV$, as the optimization problem in Eq.~(\ref{eq:nmf_obj_pca3}) with respect to $\mV$ can be decoupled into column-wise sub-problems.

For the past decade, proximal algorithm has been successfully applied to solve a wide variety of problems, such as convex optimization, non-monotone operators ~\cite{combettes2004proximal,kaplan1998proximal} with various applications to non-convex programming. More recently, proximal alternating linearized minimization (PALM) was introduced \cite{bolte2014proximal} as a linearized approximation of the proximal algorithm. % and it was in the study of variational inequalities associated to maximal monotone operators.
%In the last
%decades, it has been successfully applied to a wide variety of situations: convex optimization, nonmonotone operators \cite{combettes2004proximal,kaplan1998proximal} with various applications to nonconvex programming.
Considering the fact that objective function in Eq.~(\ref{eq:nmf_obj_pca3}) is non-convex \wrt $\mU$ and $\mV$, plus the constraints on $\mU, \mV$ is non-convex, we utilize PALM and optimize the solution as:
%\e
%\begin{aligned}
%	\label{eq:h}
%	&h(F,G) = \|X-FG_k\|_F^2+\frac{\mu}{2}||F-F_k||^2_F\\
%	&+\sum_{j}\frac{\lambda}{2}||G(:,j)-G_k(:,j)||^2 \ s.t. F \in \setF, G\in \setG
%\end{aligned}
%\ee
%with the alternating linearized minimization solutions becomes:

\e
%\begin{aligned}
\label{eq:palm1}
\mU_{k+1} = \argmin_{\mU\in\setU} \langle \mU-\mU_k, \nabla h_{\mU}(\mU_k, \mV_k) \rangle +\frac{\mu}{2}||\mU-\mU_k||^2_F,%\\
%&v_{k+1} = \argmin_{\|v\|=1,\|v\|_0 \leq s} \langle v-v_k, \nabla h_v(U_{k+1,}v_k) \rangle+\frac{\lambda}{2}||v-v_k||^2% \argmin_{\|g\|=1} \|x-F_{k+1}g\|_2^F%+\frac{\lambda}{2}||g-g_k||^2
%\end{aligned}
\ee
\e
%\begin{aligned}
\label{eq:palm}
%&U_{k+1} = \argmin_{U^TU=I} \langle U-U_k, \nabla h_U(U_k, V_k) \rangle +\frac{\mu}{2}||U-U_k||^2_F\\
\vv_{k+1} = \argmin_{\mV\in\setV} \langle \vv-\vv_k, \nabla h_{\vv}(\mU_{k+1}, \vv_k) \rangle+\frac{\lambda}{2}||\vv-\vv_k||^2,% \argmin_{\|g\|=1} \|x-F_{k+1}g\|_2^F%+\frac{\lambda}{2}||g-g_k||^2
%\end{aligned}
\ee
where $\lambda, \mu$ are tuning parameters. And as we will see later, as long as they satisfy certain conditions, the updating method above will guarantee the objective is monotonically decreasing.
\subsection{Proposed Algorithm}
To solve  Eq.~(\ref{eq:palm1}--\ref{eq:palm}), %with closed updating algorithm, 
we first derive the solution for $\mU$ by fixing $\mV$ and dividing into different cases:% when $U\in\setU_1$, followed by $U\in\setU_2$. 
\subsubsection{Optimizing $\setU=\setU_1$}
To derive the solution we first introduce the following lemma~\cite{higham1995matrix}:
\begin{lem}
	$\max \limits_{\mX^T\mX=\mId} \optr{\mX^T\mB}$ is given by $\mX=\mU\mV^T$, where $[\mU,\mSigma,\mV] =svd(\mB)$.
\end{lem}
Now we can derive the solution to Eq.~(\ref{eq:palm1}) as follows:
\begin{align}
	\mU_{k+1}&=\argmin_{\mU^T\mU=I} \langle \mU-\mU_k, \nabla h(\mU_k) \rangle+\frac{\mu}{2}||\mU-\mU_k||^2_F\nonumber\\
	%&=\argmin_{F^TF=I} \|X\|^2_F-2 \ tr (F^TXG_k^T) + \|G_k\|^2_F\\ &+\frac{\mu}{2}(\|F_k\|^2_F+\|F\|^2_F-2 \ tr (F^TF_k))\\
	&=\argmax_{\mU^T\mU=I} tr(\mU^T\mM) = \mY\mZ^T,
	\label{eq:set1UConstraint}
\end{align}
where $\mM=2(\mX-\mU_k\mV_k)V_k^T+\mu \mU_k$ and $\mY, \mZ$ is obtained from $[\mY,\mSigma,\mZ] =svd(\mM)$.
\subsubsection{Optimizing  $\setU=\setU_2$}
Similarly, if $\mU\ge 0$, then
\begin{align}
	\mU_{k+1}&=\argmin_{\mU\ge0} \langle \mU-\mU_k, \nabla h(\mU_k) \rangle+\frac{\mu}{2}||\mU-\mU_k||^2_F\nonumber\\
	%&=\argmin_{F^TF=I} \|X\|^2_F-2 \ tr (F^TXG_k^T) + \|G_k\|^2_F\\ &+\frac{\mu}{2}(\|F_k\|^2_F+\|F\|^2_F-2 \ tr (F^TF_k))\\
	&=\argmin_{\mU\ge0} \|\mU-[\mU_k+\frac{2}{\mu}(\mX-\mU_k\mV_k)\mV_k^T]\|_F^2 \\
	&= \max\{\mU_k+\frac{2}{\mu}(\mX-\mU_k\mV_k)\mV_k^T,0\}.
	\label{eq:set2UConstraint}
\end{align}

%Proof: on first hand, we have:
%\e
%\label{eq:maxF}
%	tr(F^TM) = tr(F^T u a v^T)=tr(Pa)
%\ee
%where $P=v^TF^Tu$ is orthogonal matrix $PP^T=v^TF^Tu*u^TFv=I$, thus every element including the diagonal of $P$ is no larger than 1, then we have:
%\e
%\label{eq:maxF1}
%tr(Pa)\le tr(a)
%\ee
%on the other hand, when $F=uv^T$, we have $tr(F^TM)=tr(vu^Tuav^T)=a$, thus  $F=uv^T$ is the optimized solution.

\noindent Now we turn to optimize $\vv_{k+1}$ in Eq.~(\ref{eq:palm}) given $\|\vv\|=l$:
\begin{align}
	\vv_{k+1}&= \argmin_{\mV\in\setV}\langle \vv-\vv_k, \nabla h(\vv_k) \rangle+\frac{\lambda}{2}||\vv-\vv_k||^2_F\nonumber\\
	%	&=\argmin_{\|g\|=1} \|x\|^2_F-2 <g, F_{k+1}^Tx> + \|g\|^2_F\\ &+\frac{\lambda}{2}(\|g_k\|^2_F+\|g\|^2_F-2 <g, g_{k}>)\\
	&=\argmax_{\mV\in\setV} \langle \vv,\vq \rangle,%\\
	%	&=\frac{Q}{\|Q\|_2}
	\label{eq:v}
\end{align}
where $\vq=2\mU_{k+1}^T\vx+(\lambda-2\mU_{k+1}^T\mU_{k+1}) \vv_k$.

%If $q\le \textbf{0}$, we reorder $q$ from largest to lowest as
%$\bar{q}=[q_1;q_2;\dots;q_r]$, then the optimized $v_{k+1}$ is $[1;0;\dots;0]$; otherwise, without loss of generality, we reorder $q$ from largest to lowest as
%$\bar{q}=[q_1;q_2;\dots;q_{n-1};q_n;q_{n+1}\dots;q_r]$
% with $q_1 \ge q_2 \dots \ge q_n \ge 0 > q_{n+1} \dots \ge q_r$, denote
% $\tilde{q}=[q_1;q_2;\dots;q_{n-1};q_n]$ ($\tilde{q}\ge0$), then we optimize $v_{k+1}=[\frac{\tilde{q}}{\|\tilde{q}\|};0;\dots;0]$.\footnote{Please refer to the supplementary file for derivation details} 

%To optimize $v$, first we recall that
%\begin{align}
%	v_{k+1}&= \argmin_{\|v\|=1,v \geq 0}\langle v-v_k, \nabla h(v_k) \rangle+\frac{\lambda}{2}||v-v_k||^2_F\nonumber\\
%	%	&=\argmin_{\|g\|=1} \|x\|^2_F-2 <g, F_{k+1}^Tx> + \|g\|^2_F\\ &+\frac{\lambda}{2}(\|g_k\|^2_F+\|g\|^2_F-2 <g, g_{k}>)\\
%	&=\argmax_{\|v\|=1,v \geq 0} \langle v,q \rangle,%\\
%	%	&=\frac{Q}{\|Q\|_2}
%	\label{eq:g}
%\end{align}
%where $q=2U_{k+1}^Tx+(\lambda-2) v_k$ that is equivalent to:
%\begin{equation}
%	\max_v \  J = \langle v,q \rangle \quad s.t. \ \|v\|=1,v \geq 0.
%\end{equation}
%Then we can divide into three scenarios w.r.t $q$ to optimize $v$:

\noindent In the following various cases, we first set $l=1$ and then optimize $l$ separately.
\subsubsection{Optimizing $\setV=\setV_1$}
We first consider the case that $\setV=\setV_1$, then apparently $v=\frac{\vq}{\|\vq\|}$.
\subsubsection{Optimizing $\setV=\setV_2$}
%If $\setV=\setV_2$, then we first consider $v$ lies in a sphere with $radius=1$, and then optimize $l$ separately.  

In this case, $\vv$ is positive and spherically distributed. We can optimize $\vv$ by dividing into three cases:
\begin{itemize}
	\item \textbf{Case 1}: When $\vq\ge 0$, $J=\|\vv\|\|\vq\|\cos\theta$, where $\theta$ is the angle between $\vv$ and $\vq$, obviously $\theta=0$ will maximize $\mJ$, thus we have $v^*=\frac{\vq}{\|\vq\|}$.
	\item \textbf{Case 2}: When $\vq$ has mixed signs, we can denote $\vq=[\vq_+,\vq_{-}]$, where $\vq_+>0$, while $\vq_{-}\le0$. Assume the optimized $\vv=[\vv_+,\vv_{-}]$ which corresponds to the index of $\vq$ with $\vv_{-}\neq0$ but $\vv_{-}\ge 0$, then we have:
	\begin{equation}
		\begin{aligned}
			\langle \vv,\vq \rangle&=\langle \vv_+,\vq_+ \rangle+\langle \vv_{-},\vq_{-} \rangle\\
			&<\langle \vv_+,\vq_+ \rangle
			<\langle \frac{\vv_+}{\|\vv_+\|}, \vq_+ \rangle,
		\end{aligned}
	\end{equation}
	where by following \textit{Case 1}, we have $\vv^*=[\frac{\vq_+}{\|\vq_+\|};0;\dots;0]$.
	\item \textbf{Case 3}: When  $\vq\le 0$, the objective $\mJ$ is equivalent to:
	$$
	\min \  \mJ = \langle \vv,\bar{\vq} \rangle,
	$$
	where $\bar{\vq}=-\vq$. Since $\mJ=\|\vv\|\|\bar{\vq}\|\cos\theta$, which is to maximize the angle between $\vv$ and $\bar{\vq}$, where both  $\vv$ and $\bar{\vq}$ are nonnegative. Without loss of generality, assume $\bar{\vq}$ is in increasing order, then we have:
	\begin{equation}
		\begin{aligned}
			\langle \vv,\bar{\vq} \rangle&=\sum \vv_i\bar{\vq}_i\\
			&\ge\bar{\vq}_1\sum \vv_i\\
			&\ge\bar{\vq}_1
		\end{aligned}
	\end{equation}
	where the last line follows from $(\sum \vv_i)^2\ge \sum \vv_i^2=\|\vv\|^2=1$ when $\vv\ge0$. The equation holds if and only if $\vv=[1;0\dots;0]$.
\end{itemize}

%We can validate that $v$ satisfies the constraints and maximizes the objective simultaneously.
%, and the last equation comes from Cauchy-Schwarz inequality.
\subsubsection{$\setV=\setV_3$}
WLOG, we assume $\vq=[q_1;q_2;\dots;q_r]$ is sorted in an order such that $|q_1|\ge|q_2|\ge\dots\ge|q_r|$, and if we denote $\bar{\vq}=[q_1;q_2;\dots;q_s;0;0;\dots;0]$, then $v=\frac{\bar{\vq}}{\|\bar{\vq}\|}$.
\subsubsection{$\setV=\setV_4$}
Apparently, $\setV_4$ can be regarded as a projection of $\setV_2$ to sparsity constraint. Similarly, it would divide into three cases and only choose the largest $s$ elements in $\vv$ before normalization.

\noindent Now we turn to optimize $l$. Taking the derivative \wrt $l$, by setting the gradient to be 0  we obtain:
 \begin{equation}\label{eq:l}
l=\frac{\lg \mX,\mU\mV\rg}{\lg \mU\mV,\mU\mV\rg},
 \end{equation}
 in any of the cases above. And $\mV=l\star\mV$ making all the data points reside in the sphere with $radius=l$.
 
Finally, the algorithm to solve Eq.~\eqref{eq:nmf_obj_pca3} is summarized in Alg. ~\ref{alg:alg}.

\begin{algorithm}%[tb]
	\caption{Our proposed method to solve Eq.~(\ref{eq:nmf_obj_pca3})}
	\label{alg:alg}
	\begin{algorithmic}
		\STATE {\bfseries Input:} data $\mX\in\R^{m\times n}$, rank of factors $r$, regularization parameters $\lambda,\mu$, number of iterations $K$.
		\STATE {\bfseries Initialization:} $\mU^0\in\setU, \mV^0\in\setV$,
		\FOR{$k \leq K$}
		\STATE $optimize \ \mU^{k+1} \ as \ Eq.~(\ref{eq:palm1}) \ where \  \mu\ge2\sigma_1(\mV^k\mV^{kT})$,
		\STATE $optimize \ each \ \vv^{k+1} \ as \ Eq.~(\ref{eq:palm}) \ where \  \ \lambda \ge2\sigma_1(\mU^{(k+1)T}\mU^{k+1})$,
		\STATE $optimize \ l \ as \ Eq.~(\ref{eq:l})$,
		%   \STATE $ F_k = \calP_{\setF}(F-\lambda \ (F_{k-1}G_{k-1}G_{k-1}^T-XG_{k-1}^T))$
		%   \STATE $G_k = \calP_{\setG}(G_{k-1}-\mu \ (F_k^TF_kG_{k-1}-F_k^TX))$
		\ENDFOR
		\STATE {\bfseries Output:} $\mU^K$ and $\mV^K$.
	\end{algorithmic}
\end{algorithm}

\section{Convergence Analysis}

%The whole complexity analysis involves with optimizing $U$ and $U$ in each iteration.
%In each iteration of Algorithm~\ref{alg:alg}, to optimize $U$ we first calculate $M$ with $\mathcal{O}(mnr)$ complexity. Given the size of $U$: $m\times r$, the complexity of $svd$ is $\mathcal{O}(mr^2+m^2r+r^3)$. Then the complexity to compute $U$ is $\mathcal{O}(mr^2+m^2r+r^3+mnr)$. With the consumption that for each column of $v$ is to calculate $q$ with $\mathcal{O}(mr)$ complexity, the complexity to compute $V$ is $\mathcal{O}(mnr)$. Assume that we run $K$ iterations for the loop in Algorithm~\ref{alg:alg}, the total complexity of the algorithm is $\mathcal{O}(K(mr^2+m^2r+r^3+mnr))$.
Before we start the convergence analysis, we first give the lemma below which will be critical to  the whole proof~\cite{zhou2018fenchel,liu2021factor}.
\begin{lem}
	For Eq. (\ref{eq:h}) we have:  $\nabla_{\mU}^2 h(\mU,\mV)=2\mV\mV^T, \nabla_{\mV}^2 h(\mU,\mV)=2\mU^T\mU$, and therefore:
	\begin{equation}
		\begin{split}
			&h(\mU^{k+1},\mV^k)-h(\mU^{k},\mV^k)\le \lg \mU^{k+1}-\mU^k, \nabla_{\mU} h(\mU^k,\mV^k)\rg + \frac{L_U^k}{2}\|\mU^{k+1}-\mU^k\|^2_F,\\
			&h(\mU^{k+1},\mV^{k+1})-h(\mU^{k+1},\mV^k)\le \lg \mV^{k+1}-\mV^k, \nabla_{\mV} h(\mU^{k+1},\mV^k)\rg + \frac{L_V^k}{2}\|\mV^{k+1}-\mV^k\|^2_F,
		\end{split}
	\end{equation}
where $L_U^k=\sigma_1(\nabla_{\mU}^2 h(\mU,\mV))=2\sigma_1(\mV^k\mV^{kT}), L_V^k=\sigma_1(\nabla_{\mV}^2 h(\mU,\mV))=2\sigma_1(\mU^{(k+1)T}\mU^{k+1})$ and $\sigma_1(\mX)$ denotes the largest eigenvalue of $\mX$.
\end{lem}
To analyse  the convergence, we rewrite Eq. \eqref{eq:h} as
\e
\min_{\mU,\mV} f(\mU,\mV)= h(\mU,\mV) + \delta_{\setU}(\mU) + \delta_{\setV}(\mV),
\label{eq:obj no regularizer}\ee
where $\delta_{\setU}(\mU) = \left\{ \begin{matrix} 0, & \mU\in\setU\\ \infty, & \mU\notin \setU\end{matrix}\right.$ is the indicator function of the set $\setU$ and therefore nonsmooth, so is $\delta_{\setV}(\mV)$.

The following result establishes that the subsequence convergence property of the proposed algorithm.% \ie, the sequence generated by Algorithm~\ref{alg:alg} is bounded and any of its limit point is a critical point of Eq. \eqref{eq:obj no regularizer}.
\begin{thm}%[Subsequence convergence]
	Let $\{\mW_k\}_{k\geq 0} =\{(\mU_k,\mV_k)\}_{k\geq 0}$ be the sequence generated by Algorithm~\ref{alg:alg} with regularization parameter $\lambda^k\ge L_V^{k-1}, \mu^k\ge L_U^{k-1}$, then the objective in Eq. (\ref{eq:h}) is monotonically non-increasing.
	\label{thm:subsequence convergence}\end{thm}
%\end{lem}
\begin{proof}
	First note that for all $k$, according to Algorithm \ref{alg:alg}, we always have $\delta_{\setU}(\mU_k) = \delta_{\setV}(\mV_k) = 0$ and accordingly $f(\mU^k,\mV^k) = h(\mU^k,\mV^k)$.
%	
%	Since $h(U,V)$ has Lipschitz continuous gradient at $U\in \setU,V\in\setV$ with Lipschitz gradient $L_c$ and $\lambda>L_c$, we define $h_{L_c}(U,U',V)$ as proximal regularization of $h(U,V)$ linearized at $U',V$:
%	\[
%	h(U',V) + \langle \nabla_U h(U',V),U - U' \rangle + \frac{L_c}{2}\|U - U'\|_F^2,
%	\]
	By the definition of Lipschitz continuous gradient and Taylor expansion,
	we have
	\e
	h(\mU^{k},\mV^{k-1})-h(\mU^{k-1},\mV^{k-1})\le \lg \mU^{k}-\mU^{k-1}, \nabla_{\mU} h(\mU^{k-1},\mV^{k-1})\rg + \frac{L_U^{k-1}}{2}\|\mU^{k}-\mU^{k-1}\|^2_F.
	\label{eq:Lip consequence}\ee
	Also by the definition of proximal map, we get:
	\e\begin{split}
		\mU_{k} %= \calP_{\setF} (F_{k-1} - \lambda \nabla_F h(F_{k-1},G_{k-1}))\\
		%& = \argmin_{F\in \setF}\| F - (F_{k-1} - \lambda \nabla_F h(F_{k-1},G_{k-1}))  \|_F^2\\
		= \argmin_{\mU} \delta_{\setU}(\mU) +  \frac{\mu^k}{2}\|\mU - \mU_{k-1}\|_F^2+\langle \nabla_{\mU} h(\mU_{k-1},\mV_{k-1}),\mU - \mU_{k-1} \rangle.%\\
		%& =  \argmin_{F\in \setF}h_\lambda(F,F_{k-1},G_{k-1}),
	\end{split}\label{eq:opt F}\ee
	Therefore, we have $f(\mU_k,\mV_{k-1}) \le f(\mU_{k-1},\mV_{k-1})$, which implies the following:
	\e
	%\begin{aligned}
		\delta_{\setU}(\mU_k) +  \frac{\mu^k}{2}\|\mU_k - \mU_{k-1}\|_F^2 
		+\langle \nabla_{\mU} h(\mU_{k-1},\mV_{k-1}),\mU_k - \mU_{k-1} \rangle\leq \delta_{\setU}(\mU_{k-1}).
		\label{eq: h lambda t}
	%\end{aligned}
	\ee
	%	\e
	%	h_\lambda(F_{k},F_{k-1},G_{k-1}) \leq h(F_{k-1},G_{k-1}).
	%	\label{eq: h lambda to h}\ee
	Combining Eq.~(\ref{eq:Lip consequence}) to Eq.~(\ref{eq: h lambda t}), we have:
	\e\label{eq:FG_update}
	\begin{aligned}
		&h(\mU_{k},\mV_{k-1}) + \delta_{\setU}(\mU_k)\\
		 \leq& h(\mU_{k-1},\mV_{k-1})+ \langle \nabla_{\mU} h(\mU_{k-1},\mV_{k-1}),\mU_k - \mU_{k-1} \rangle 
		+ \frac{L_U^{k-1}}{2}\|\mU_k - \mU_{k-1}\|_F^2+\delta_{\setU}(\mU_{k})\\
		\leq & h(\mU_{k-1},\mV_{k-1})+\frac{L_U^{k-1}}{2}\|\mU_k - \mU_{k-1}\|_F^2
		+\delta_{\setU}(\mU_{k-1})-\frac{\mu^k}{2}\|\mU_k - \mU_{k-1}\|_F^2\\
		=&h(\mU_{k-1},\mV_{k-1}) + \delta_{\setU}(\mU_{k-1}) - \frac{\mu^k - L_U^{k-1}}{2}\|\mU_{k} - \mU_{k-1}\|_F^2.
	\end{aligned}
	\ee
	Similarly, we have
	\e
		\label{eq:G_update}
		h(\mU_k,\mV_{k}) - h(\mU_{k},\mV_{k-1})+\delta_{\setV}(\mV_k)-\delta_{\setV}(\mV_{k-1})
		\leq -\frac{\lambda^k - L_V^{k-1}}{2}\|\mV_{k} - \mV_{k-1}\|_F^2,
\ee
which illustrates within each update, the objective is non-increasing if $\lambda^k>L_V^{k-1}, \mu^k> L_U^{k-1}$. 
\end{proof}
Global sequence convergence property and at least sub-linear convergence rate can be proved by following~\cite{bolte2014proximal,zhu2018dropping}.
\section{Experiment}
We will carefully examine the effectiveness of our proposed method including spherical nonnegative matrix  factorization (on synthetic data) and PCA (real-world data) respectively.
\begin{figure*}[t!]
	\centering
	%\label{fig:compare}
	\subfigure{
		\begin{minipage}[b]{0.31\textwidth}
			\includegraphics[height=4.5cm,width=1\textwidth]{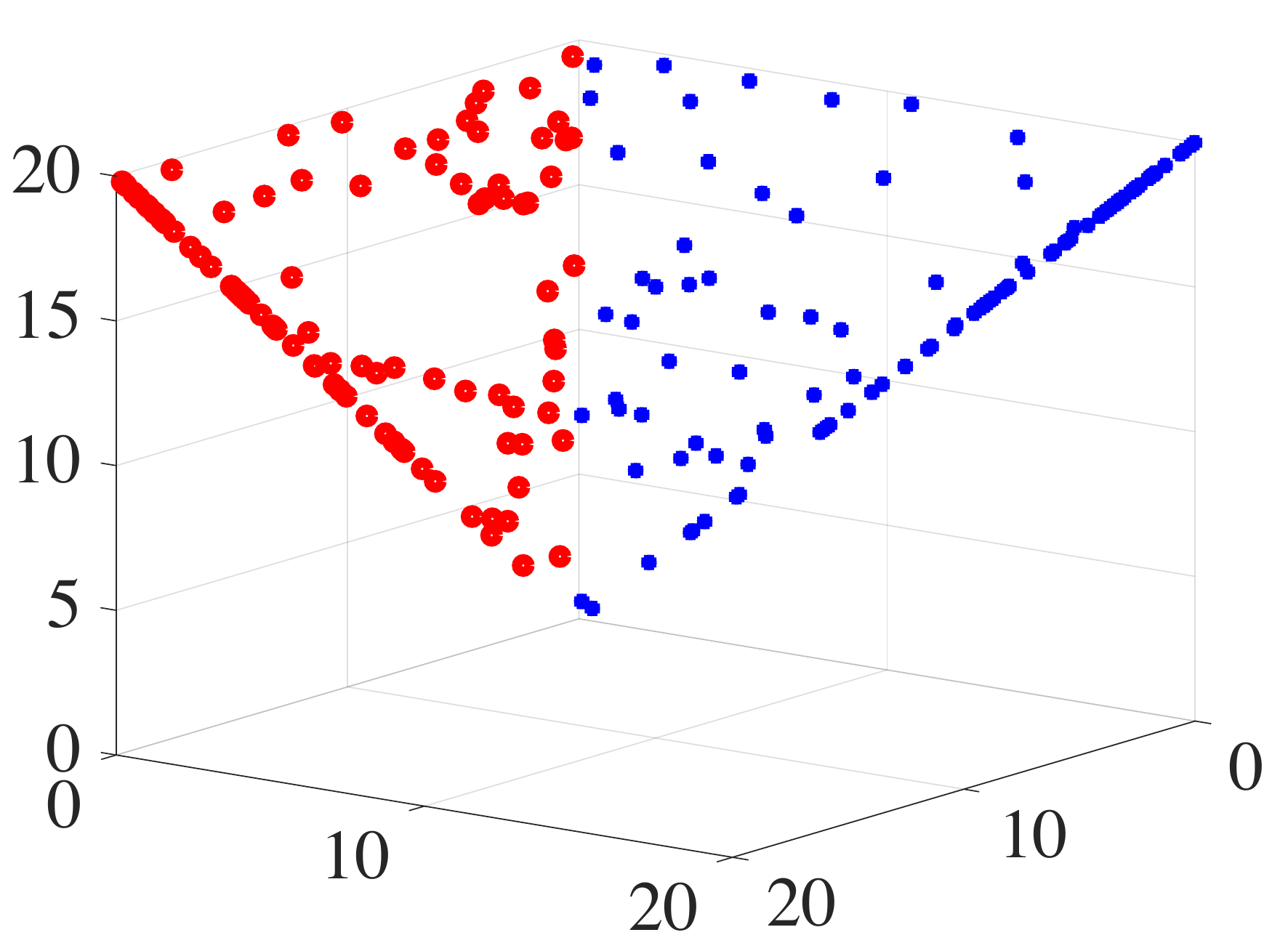}
		\end{minipage}
	}
	\subfigure{
		\begin{minipage}[b]{0.31\textwidth}
			\includegraphics[height=4.5cm,width=1\textwidth]{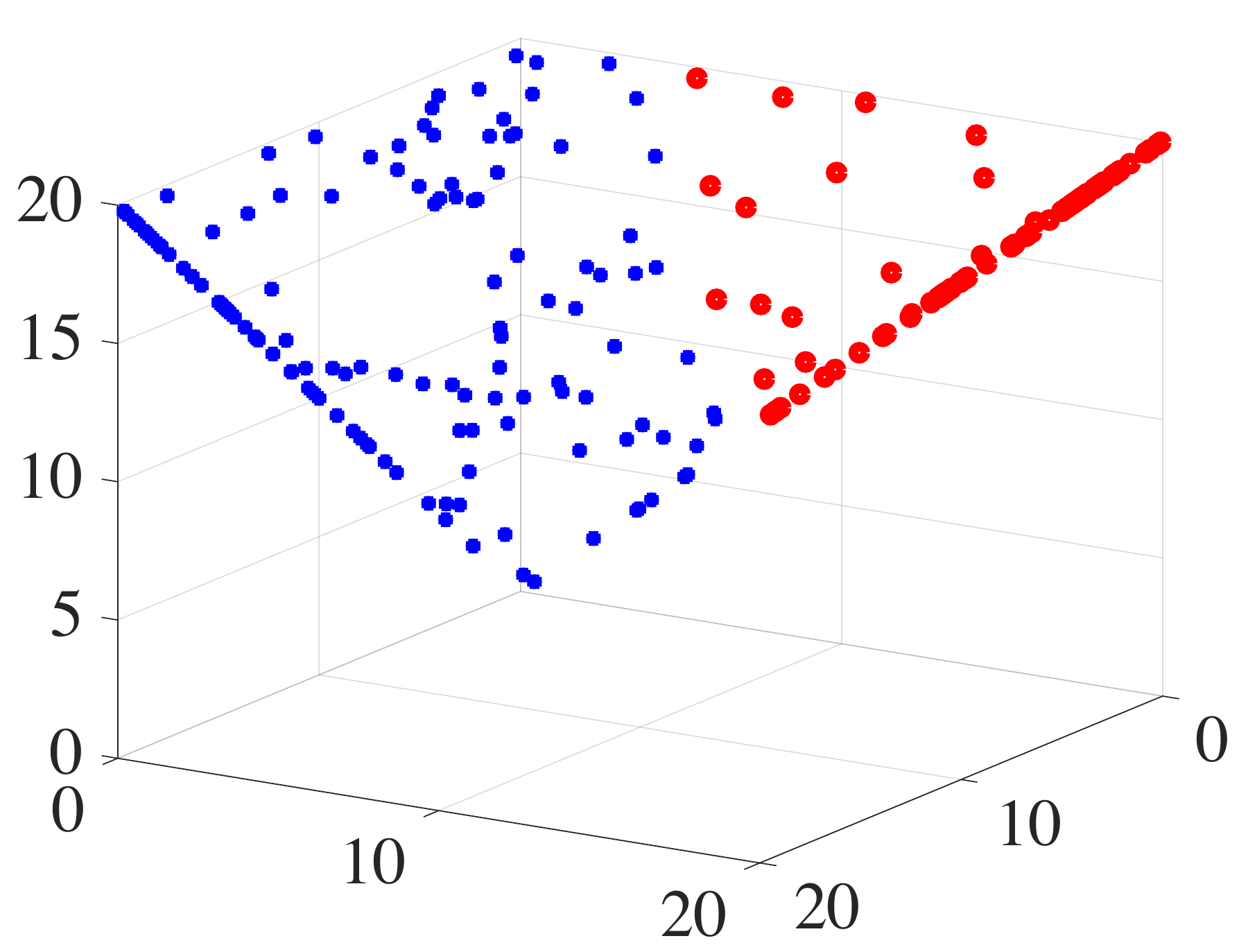}
		\end{minipage}
	}
	\subfigure{
		\begin{minipage}[b]{0.31\textwidth}
			\includegraphics[height=4.5cm,width=1\textwidth]{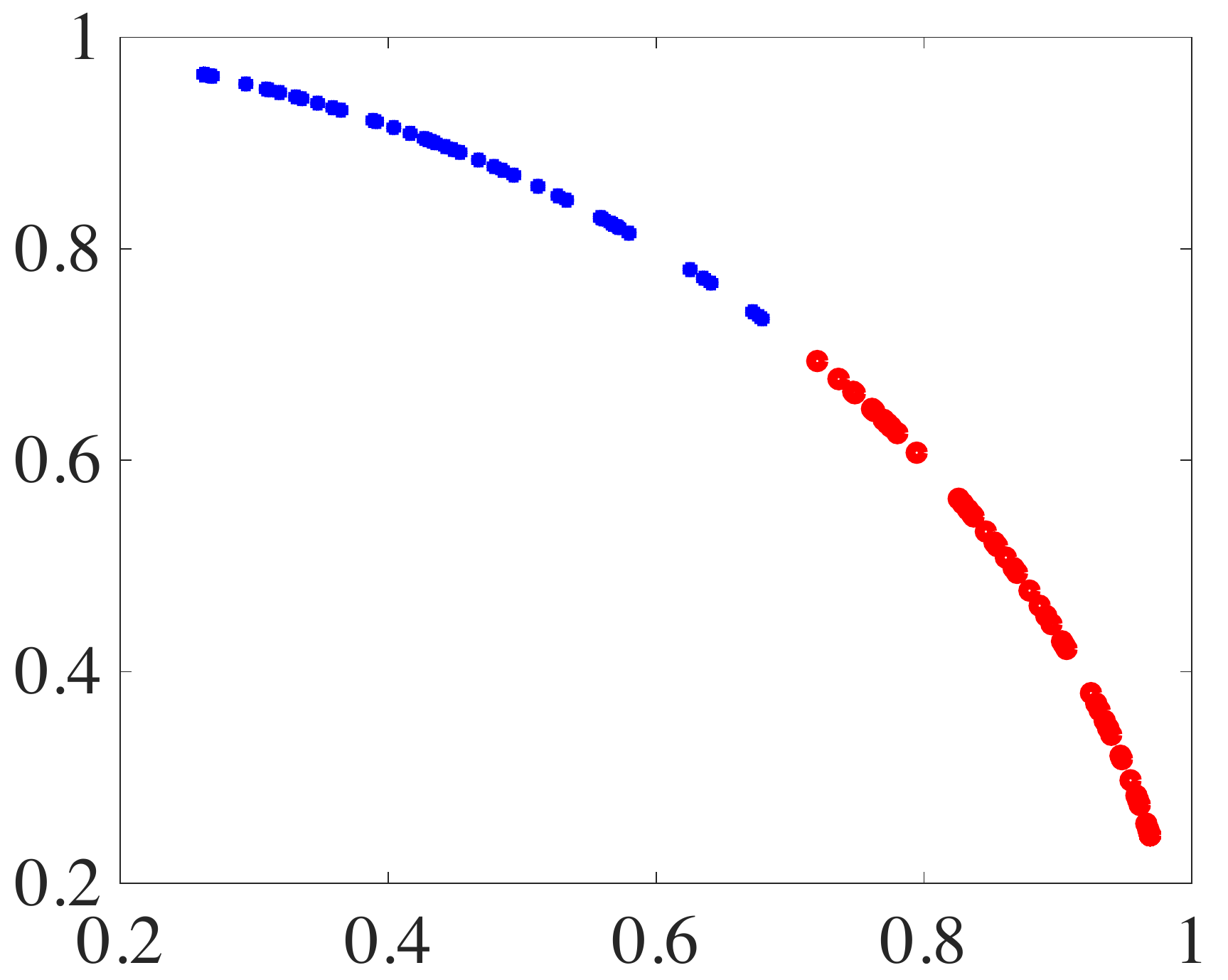}
		\end{minipage}
	}
	\caption{\textbf{Left}: two groups of data generated from two angles. \textbf{Center}: clustering result with distance-based method $K$-means. \textbf{Right}: projection into $\mV$ via factorization $\mX\approx\mU\mV$ where $\mU\in\setU_1, \mV\in\setV_2$. Blue and red represent different clusters~\cite{liu2019spherical}.}
	\label{fig:toy}
\end{figure*}
\subsection{Synthetic Data}
\begin{figure*}[h!]
	\centering
	%\label{fig:compare}
	\subfigure{
		\centering
		\begin{minipage}[b]{0.32\textwidth}
			\includegraphics[ width=\linewidth,height=4.5cm]{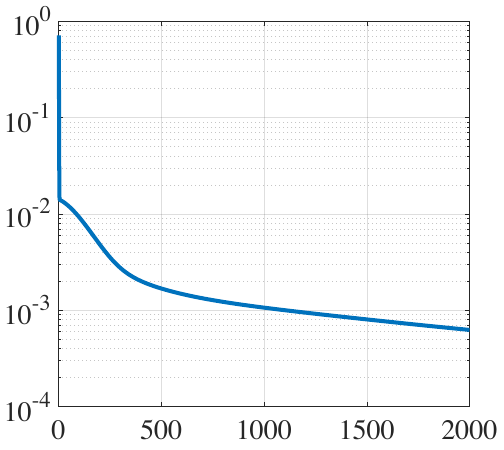}
		\end{minipage}
	}
	\subfigure{
		\centering
		\begin{minipage}[b]{0.31\textwidth}
			\includegraphics[width=\linewidth,height=4.5cm]{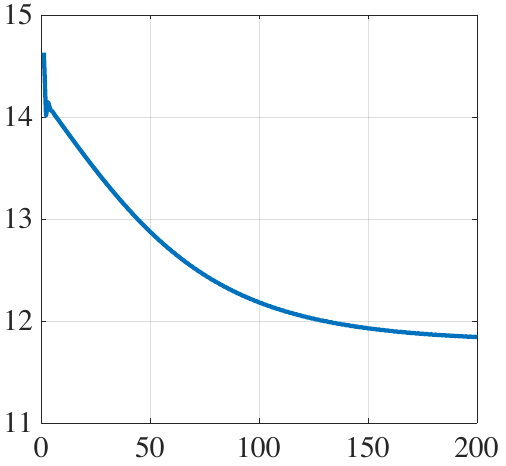}
		\end{minipage}
	}
	\subfigure{
		\centering
		\begin{minipage}[b]{0.31\textwidth}
			\includegraphics[width=\linewidth,height=4.5cm]{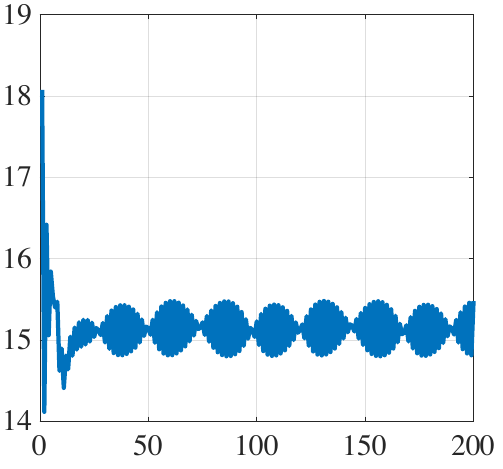}
		\end{minipage}
	}
	\caption{\textbf{Left}: Typical $\|\mU_{k+1}-\mU_{k}\|_F$ with updates. \textbf{Center}: Objective  converges when $\lambda$ is sufficiently large. \textbf{Right}: Objective oscillates when $\lambda$ is relatively small.}
	\label{fig:change}
\end{figure*}
We first generate 200 data points, half of which is distributed within the region between $\mX=\mZ$ and $\mZ$ axis (denoted as blue dots in the left part of Fig.~\ref{fig:toy}, while another group is generated within the region between $\mY=\mZ$ and $\mZ$ axis (denoted as the red dots). These two clusters of data are generated through different angles. Thus when we do clustering, it should be angle distance rather than Euclidean distance to determine the clustering result. For our method, we learn a feature matrix $\mU\in \R^{3\times 2}$ and plot the component matrix $\mV\in \R^{2\times 200}$ as the Right part illustrates. We see that, Euclidean distance-based method (such as $K$-means) will yield poor clustering result (middle part), while ours will obtain good clustering result (due to the data's generation method, they can be separated by anti-diagonal).

%\begin{figure*}
%	\centering
%	\includegraphics{loss.pdf}
%	\caption{\textbf{Left}: $\|U_{k+1}-U_{k}\|_F$ with updates. \textbf{Center}: Objective function converges when $\lambda$ is sufficiently large. \textbf{Right}: Objective function oscillates when $\lambda$ is small. Our research gives a guidance to determine proper $\lambda$.}
%	\label{fig:change}
%\end{figure*}

Also, we show the convergence of $\{\mW_k\}_{k\geq 0} =\{(\mU_k,\mV_k)\}_{k\geq 0}$ generated by our method. As Fig.~\ref{fig:change} shows, after short iterations, the generated sequences will be stable, which is in accordance with the convergence proof. Fig.~\ref{fig:change}  also illustrates that the objective with update. We also see that only sufficiently large $\lambda$, $\mu$ will make objective monotonically decreasing. 

\subsection{Real-world Datasets Experiment}
\begin{figure*}[h!]
	\centering
	%\label{fig:compare}
	\subfigure{
		\centering
		\begin{minipage}[b]{0.315\textwidth}
			\includegraphics[ width=\linewidth,height=7cm]{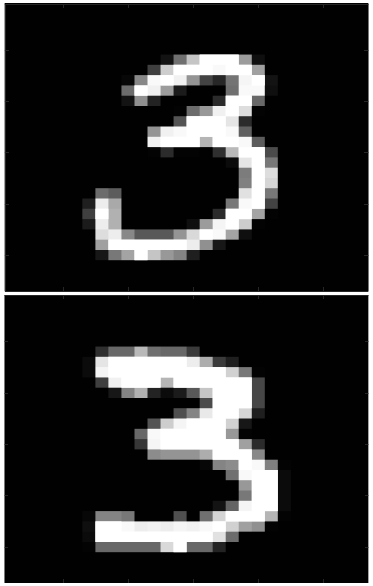}
		\end{minipage}
	}
	\subfigure{
		\centering
		\begin{minipage}[b]{0.315\textwidth}
			\includegraphics[width=\linewidth,height=7cm]{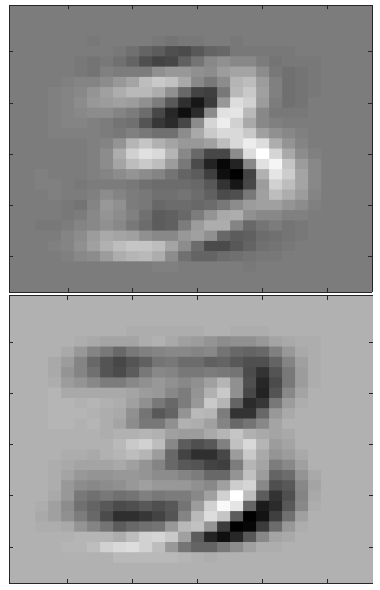}
		\end{minipage}
	}
	\subfigure{
		\centering
		\begin{minipage}[b]{0.32\textwidth}
			\includegraphics[width=\linewidth,height=7cm]{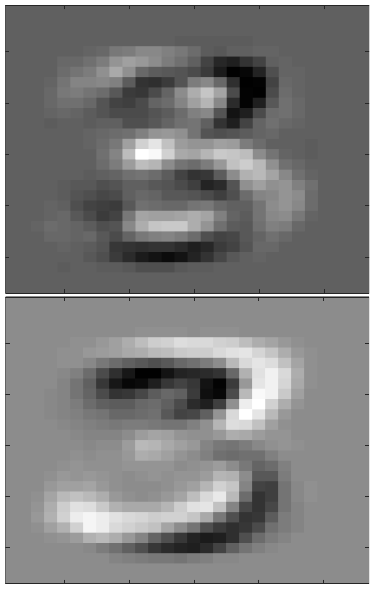}
		\end{minipage}
	}
	\caption{\textbf{Left}: original images $\vx$;  \textbf{Center}: $\vu$ corresponds to the largest coefficient; \textbf{Right}: $\vu$ corresponds to second largest coefficient. In this setting, $\mU\in\setU_1, \mV\in\setV_4$.}
	\label{fig:four}
\end{figure*}
\begin{figure}
	\centering
	\includegraphics[width=.45\linewidth]{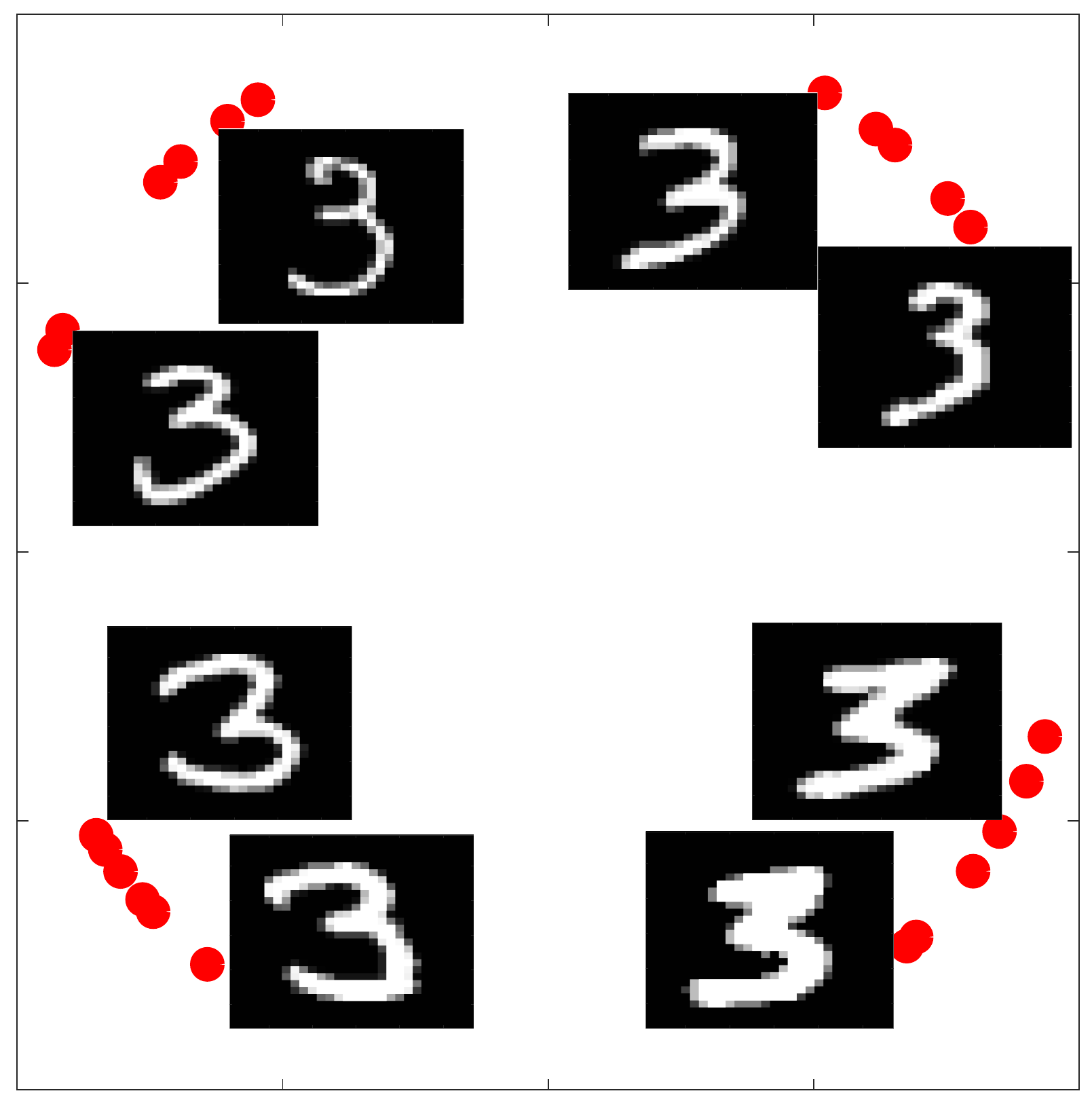}
	\caption{All data point is distributed in the sphere, we see the images with similar pattern lie close to each other. From left to right, the tail becomes shorter while from bottom to up the size turns smaller.}
	\label{fig:five}
\end{figure}
In this subsection, we will first test our proposed method on \textit{MNIST handwritten digits} dataset~\cite{lecun1998gradient}. We collect all images of digit three from the test set image file which give 1010 images. The original image is $28 \times 28$, which corresponds to $m = 784$ dimension feature. We set $r = 10$, that is to project into 10 dimensional space, then $\mU$ is made up with 10 mutually orthogonal vectors. Each image data in the new coordinate system is represented as $\vv$. To demonstrate the effectiveness in light of visualization, we set $s = 2$, that is every original data is approximated by the linear combination of top 2 features in $\vu$. Fig. \ref{fig:four} illustrates the most critical features. We see the first learned component has very similar patten as original image. We plot these data where $\vu_1$ and $\vu_2$ are the principal component vectors as Fig. \ref{fig:five} demonstrates, where $\mX$-axis and $\mY$-axis corresponds to the first and second element in $\vv$. The first and second row image in Fig. \ref{fig:four} lie in the second and third quadrant in Fig. \ref{fig:five}, where they share the similar first component but varies in second component. We see that our proposed method can yield reasonable and promising results even we project from original high dimension space into very low one.
\bibliography{main}
\bibliographystyle{ieeetr}
\end{document}